\RequirePackage{fix-cm}
\documentclass[twocolumn]{svjour3} 
\newcommand{\argmin}{\arg\!\min}
\smartqed
\usepackage[misc]{ifsym}
\usepackage{amsmath}
\usepackage{amssymb}
\usepackage{graphicx}
\usepackage{ucs}
\usepackage{graphics}
\usepackage{epsfig} 
\usepackage{epstopdf}
\usepackage[utf8x]{inputenc}
\usepackage{algorithm}
\usepackage{tikz}
\usepackage[outline]{contour}
\usepackage{algorithmic}
\usepackage[T1]{fontenc}
\usepackage{authblk}
\usepackage{booktabs}
\usepackage{multirow}
\usepackage[justification=centering]{caption}
\begin{document}
%\mainmatter

\title{LogDet Rank Minimization with Application to Subspace Clustering}

\author{
Zhao Kang\and Chong Peng \and Jie Cheng\and Qiang Cheng}
\institute{Z. Kang\and C. Peng\and Q. Cheng   \at Computer Science Department, Southern Illinois University, Carbondale, IL, USA\\
\email{qcheng@cs.siu.edu   
\and J. Cheng  \at Department of Computer Science and Engineering, University of Hawaii at Hilo, Hilo, HI, USA
}}

\date{Received: date / Accepted: date}

\maketitle

\begin{abstract}
Low-rank matrix is desired in many machine learning and computer vision problems. Most of the recent studies use the nuclear norm as a convex surrogate of the rank operator. However, all singular values are simply added together by the nuclear norm, and thus the rank may not be well approximated in practical problems. In this paper, we propose to use a log-determinant (LogDet) function as a smooth and closer, though non-convex, approximation to rank for obtaining a low-rank representation in subspace clustering. Augmented Lagrange multipliers strategy is applied to iteratively optimize the LogDet-based non-convex objective function on potentially large-scale data. By making use of the angular information of principal directions of the resultant low-rank representation, an affinity graph matrix is constructed for spectral clustering. Experimental results on motion segmentation and face clustering data demonstrate that the proposed method often outperforms state-of-the-art subspace clustering algorithms.    
\keywords{Matrix rank approximation\and Subspace clustering\and Nuclear norm\and  Log-determinant\and Low-rank representation\and Angular information\and Segmentation}
\end{abstract}

%%%%%%%%%%%%%%%%%%%%%%%%%%%%%%%%%%%%%%%%%%%%%%%%%%
\section{Introduction}
Matrix rank minimizing~\cite{fazel2002matrix} is ubiquitous in machine learning, computer vision, control, signal processing and system identification. For instance, low-rank representation based subspace clustering \cite{liu2010robust,liu2013robust,vidal2014low} and matrix completion \cite{candes2009exact,hu2013fast} methods have achieved great success recently. Subspace clustering \cite{vidal2010tutorial} is one of the fundamental topics with numerous applications, e.g., image representation \cite{eldar2009robust,yang2008unsupervised}, face clustering \cite{elhamifar2013sparse,liu2013robust}, and motion segmentation \cite{rao2010motion,lauer2009spectral}. It is assumed that high-dimensional data is more likely a union of low-dimensional subspaces rather than one individual subspace. For example, different subspaces are needed to describe trajectories of different moving objects in a video sequence. Subspace clustering is an intrinsically difficult problem, since we need to simultaneously cluster all data points into multiple groups and find a low-dimensional subspace fitting each group of points. 

Subspace clustering has been an active research topic over the past decades. Four main categories of methods are proposed \cite{elhamifar2013sparse}: iterative, algebraic, statistical, and spectral clustering-based methods. The first three kinds of approaches are sensitive to initialization, noise and outliers; in addition, they are difficult to optimize \cite{elhamifar2013sparse}. Spectral clustering-based methods have achieved promising performance, where the key is to learn a good affinity matrix of data points. For instance, the algorithms of local subspace affinity (LSA) \cite{yan2006general}, locally linear manifold clustering (LLMC) \cite{goh2007segmenting}, and spectral local best-fit flats (SLBF) \cite{zhang2012hybrid}, use local information around each point to construct the affinity matrix, while spectral curvature clustering (SCC) \cite{chen2009spectral} method preserves the global structures of the whole data set in deriving the affinity matrix. Subsequently, K-means \cite{jing2007entropy} or Normalized Cuts (NCut) \cite{shi2000normalized,von2007tutorial} are applied to the affinity matrix to obtain clustering results. 

Recently, some spectral clustering based methods, such as sparse representation (SSC) \cite{elhamifar2013sparse}, low-rank representation (LRR) \cite{liu2013robust}, have been proposed to obtain state-of-the-art results in subspace clustering. SSC represents each data point as a sparse linear combination of the other points and solves an $l_1$-norm regularized minimization problem for sparsity.  SSC shows promising results if the subspaces are either independent or disjoint \cite{elhamifar2010clustering}. 

The basic idea of LRR is to learn a low-rank representation of data by capturing the global Euclidean structure of the whole data. In this scheme, each data point is represented as a linear combination of the examples in the data matrix itself, and a convex nuclear norm minimization is used as a surrogate of the rank function to obtain the desired low-rank representation. Though its optimization is well-studied and has a global optimum, its performance may be far from optimal in real applications because the nuclear norm might not be a good approximation to the rank function. Compared to the rank function to which all nonzero singular values have equal contributions, the nuclear norm treats those values differently by simply adding them together. As a result, the nuclear norm may be dominated by a few very large singular values and significantly deviated from the true value of the rank. Several papers have considered this problem of using the nuclear norm and designed methods to alleviate it by either thresholding or removing some of the singular values; for instance, singular value thresholding \cite{cai2010singular} and truncated nuclear norm \cite{hu2013fast} both considerably enhance the performance of matrix completion. 

In this paper, we propose to use a log-determinant (LogDet) function for rank approximation and study its minimization in subspace clustering. Different from the nuclear norm-based approaches which minimize the summation of all  singular values, our approach aims to minimize the rank by making the contribution to be much closer to one from a big singular value, while zero from a small singular value. In this way, we can get closer and more robust approximation to the rank function than the nuclear norm. Since the LogDet function is non-convex, we apply the method of augmented Lagrange multipliers (ALM) to solve the associated optimization for potentially large-scale applications, in which the subproblem for minimizing the LogDet function in each iteration has a closed-form solution. To demonstrate the effectiveness of our LogDet minimization\\ 
method, we apply it to subspace clustering. By employing a rather simple formulation based on the LogDet function, we obtain a low-rank representation for subspace clustering. Subsequently, we exploit the angular information of principal directions of such a representation to further enhance the separation ability of the affinity matrix.
In summary, our main contributions of this work include:
\begin{itemize}
\item{More accurate and robust rank approximation is used to obtain the low-rank  representation, which is able to capture the global structure of the dataset.}
\item{An iterative optimization algorithm is designed for minimizing this rank approximation-based objective function. Theoretical analysis shows that our algorithm converges to a stationary point. Specifically, the proposed optimization method is applied to subspace clustering. }
\item{Angular information of principal directions of the low-rank representation is employed to further exploit the intrinsic local geometrical structure relevant to the membership of data points. }
\item{Extensive experiments demonstrate the effectiveness of the proposed LogDet minimization method for rank approximation. Especially, when used for subspace clustering, our simple formulation shows favorable performance  compared to other state-of-the-art methods, although we do not explicitly account for outliers in our model. This demonstrates the robustness of our approach. }
\end{itemize}
The remainder of the paper is organized as follows: Section \ref{relatedwork} provides a brief review of LRR and SSC. In Section \ref{proposedmodel}, we present the proposed approximation and design an efficient optimization scheme. We give convergence analysis in Section \ref{convergence}. Experimental results are shown in Section \ref{experiment}. Finally, conclusions are drawn in Section \ref{conclusion}.

\section{Review of LRR and SCC}
\label{relatedwork}
In this section, we give a brief review of SSC and LRR.

Let $X=[x_1, x_2, ..., x_n]\in \mathbf{\mathcal{R}}^{d\times n}$ be a set of $d$-dimensional data points drawn from an unknown union of $k$ linear subspaces $S_1, S_2, ... , S_k$. The task of subspace clustering is to segment data points into $k$ subspaces.
 
LRR tries to seek the lowest rank representation among many possible linear combinations of the bases in a given dictionary, which typically is the data matrix itself. The problem can be formulated as:
\begin{equation}
\min_Z\hspace{.2cm} rank(Z) \hspace{.2cm}  s.t. \hspace{.2cm} X=XZ,
\end{equation}
where $Z=[z_1, z_2, ..., z_n]$ is the coefficient matrix with each $z_i$ being the representation of $x_i$. The above problem is NP-hard due to the combinatorial nature of the rank function.

The tightest convex relaxation of the rank function \cite{recht2010guaranteed} is the nuclear norm. For a matrix $D \in {\mathcal{R}}^{m \times n}$, its nuclear norm is defined as $\left\|D\right\|_*=\sum_{i=1}^{\min(m, n)}\sigma_i(D)$, where $\sigma_i(D)$ means the $i$-th singular value of $D$. Using this relaxation, LRR solves the following problem:
 \begin{equation}
\min_Z \left\|Z\right\|_* \hspace{.4cm}   s.t. \hspace{.4cm} X=XZ. 
\end{equation}
After obtaining $Z$, the affinity matrix $W$ is defined as
 \begin{equation}
W=|Z|+|Z^T|.
\end{equation}
Then the spectral clustering algorithm, Normalized Cuts \cite{shi2000normalized} is used to produce the final segmentation.

SSC aims to find a sparse representation of $X$ by solving the following convex optimization problem:
\begin{equation}
\begin{split}
&\min_{Z, E, S} \|Z\|_1+\frac{\alpha}{2}\|E\|_F^2+\gamma\|S\|_1, \\
 &s.t. \hspace{.2cm}X=XZ+E+S,\quad diag(Z)=0,
\end{split}
\end{equation}
where $\|S\|_1=\sum_{ij} |S_{ij}|$, $S$ is a sparse matrix containing the gross error, and $\|E\|_F^2=\sum_i\sum_j E_{ij}^2$, $E$ is a matrix of fitting residuals.
After obtaining $Z$, subsequent procedures are similar to LRR.

\section{LogDet Rank Approximation and Its Minimization Algorithm}
\label{proposedmodel}

A function $f: \mathbf{\mathcal{R}}^{n}\rightarrow[-\infty, \infty]$ is absolutely symmetric if $f(x)$ is invariant under arbitrary permutations and sign changes of the elements of $x$. Based on this function $f(x)$, we have the following theorem \cite{lewis1995convex}.
\begin{theorem}
Function $F: \mathbf{\mathcal{R}}^{n_1\times n_2}\rightarrow \mathbf{\mathcal{R}}$ is unitarily invariant if $F(X)=f(\sigma(X))=f \circ \sigma(X)$, where $X\in\mathbf{\mathcal{R}}^{n_1\times n_2} $ whose singular value decomposition is\\
 $X=U diag(\lbrace\sigma_i\rbrace_{1\leq i \leq n}) V^T$, $\sigma(X):  \mathbf{\mathcal{R}}^{n_1\times n_2}\rightarrow  \mathbf{\mathcal{R}}^{n}$ are singular values of $X$, and $n=min(n_1, n_2)$. And the gradient of $F(X)$ at $X$ is
\begin{equation}
\frac{\partial F(X)}{\partial X}=U diag(\theta) V^T,
\label{theorem}
\end{equation}
where $\theta=\frac{\partial f(y)}{\partial y}|_{y=\sigma (X)}$.
\end{theorem}
Equation (\ref{theorem}) can be obtained directly from Theorem 3.1 of \cite{lewis1995convex}.

In this work, we utilize unitarily invariant function LogDet to achieve a closer, though not convex, rank relaxation than the nuclear norm. We apply the method of ALM for LogDet rank approximation associated minimization. To explain our method, we specifically consider using LogDet as a rank surrograte in subspace clustering. We first obtain a low-rank representation of high-dimensional data based on the LogDet optimization. Then we construct an affinity graph matrix for spectral clustering by using the angular information of principal directions of the low-rank representation.
\subsection{LogDet rank minimization}

We use $\log det ( I +  Z^T Z)$ as a surrogate of the rank function of $Z$. It is obvious that  
$\log det(I+Z^T Z) = \sum_{i=1}^{n} \log (1+ \sigma_i^2 (Z)) $. Because it can be easily verified that 
$\log(1+ \sigma_i^2(Z) ) \le \sigma_i(Z)$ for any $\sigma_i(Z) \ge 0$, we always have 
$\log det (I + Z^T Z) \le \left\|Z\right\|_*$; especially, if there are large nonzero singular values, the LogDet function will be much smaller than the nuclear norm since $\log(1+ \sigma_i^2(Z) ) \ll \sigma_i(Z)$ for a large $\sigma_i(Z) > 1$. 
It is noted that for small nonzero singular values, their contribution to the LogDet function will be significantly reduced compared to the nuclear norm. Because small nonzero singular values are often regarded as being from noise in the data, the LogDet function reduces noise effect more compared to the nuclear norm.  

It is worthwhile to note that a similar function \\
$\log det(X+\delta I)$ was proposed in \cite{fazel2003log} to approximate rank and iterative linearization was used to find a local minimum. However, $\delta$ is a very small constant (e.g., $10^{-6}$), which leads to biased approximation for small singular values. 

This LogDet function is differentiable with respect to the singular values by Theorem 1, and even though it is non-convex, its minimization is rather simple by using our optimization method. 
To explain its minimization, we consider its specific application to subspace clustering. By employing the above LogDet function, we simply formulate the subspace clustering into the following unconstrained nonconvex minimization problem:  
\begin{equation}
\min_Z \hspace{.1cm} \log\hspace{.1cm} det(I+Z^TZ)+\rho \left\|X-XZ\right\|_F^2,
\label{problemd}
\end{equation}
where $I\in \mathbf{\mathcal{R}}^{n\times n}$ is the identity matrix. The first term of (\ref{problemd}) is to minimize the rank of $Z$, while the second is a relaxation of $X = XZ$, which is referred to as the self-expressiveness of $X$ with $Z$ representing the similarity between data points. Because the LogDet function is not convex in $Z$, we resort to ALM technique to solve (\ref{problemd}), by re-writing (\ref{problemd}) as follows:
\begin{equation}
\label{newproblem}
\min_Z \hspace{.1cm} \log\hspace{.1cm} det(I+Z^TZ)+\rho \left\|X-XW\right\|_F^2 \hspace{.1 cm} s.t. \hspace{.1cm} Z=W.
\end{equation}

We turn to minimizing the following  augmented Lagrangian function:
\begin{equation}
\label{newform}
\begin{split}
L(Y, Z, W, \beta)=\log\hspace{.1cm} det(I\!+Z^TZ)\!+\rho \left\|X-XW\right\|_F^2\\\!+\frac{\beta}{2}\left\|Z-W\right\|_F^2\!+Tr(Y^T(Z-W)),
\end{split}
\end{equation} 		 
where $\beta>0$ is a penalty parameter and $Y$ is the Lagrangian dual variable. With a sufficiently large $\beta$, the objective function converges to objective function in (\ref{problemd}). This can be solved by updating $Z$, $W$, and $Y$ alternatively while fixing the other variables. Specifically, assume at the $k$th iteration we have obtained $Z^k, W^k$, and $Y^k$, then for the $(k+1)$th iteration, the optimization problem (\ref{newform}) can be updated via the following four steps.

\begin{algorithm}[t]
\caption{\textbf: LogDet Rank Minimization}
\textbf{Input:} data matrix $X$, parameters $\rho>0, \gamma>1, \text{and } \beta_0>0$. \\
\textbf{Initialize:} $Z=I\in \mathbf{\mathcal{R}}^{n\times n}$, $Y=0$.\\
\textbf{Repeat}

\begin{algorithmic}[1]
\STATE Update $W$ as:
\begin{equation*}
W^{k+1}=(\beta_k I+2\rho X^TX)^{-1}(2\rho X^TX+Y^k+\beta_k Z^{k+1}).
\end{equation*}
\STATE Solve $Z$ using (\ref{variablez}) and (\ref{svdf}).
\STATE Update the augmented multiplier Y and the augmented Lagrange multiplier $\beta$:
\begin{align*}
Y^{k+1}&=Y^{k}+\beta_k(Z^{k+1}-W^{k+1}),\\
\beta_{k+1}&=\gamma \beta_k.
\end{align*}

\end{algorithmic}
\textbf{Until} stopping criterion is satisfied.\\
\textbf{Return} $Z^{*}=Z^{k+1}$.
\end{algorithm}

Step 1: Computing $W^{k+1}$. Fix $Z^{k}$ and $Y^k$ and then calculate $W^{k+1}$:
\begin{equation}
\begin{split}
W^{k+1}=\argmin_W \rho\left\|X-XW\right\|_F^2+\\
\frac{\beta_k}{2}\left\|Z^{k}-(W-\frac{1}{\beta_k}Y^k)\right\|_F^2,
\end{split}
\end{equation}  
which has a closed-form solution, 
\begin{equation}
W^{k+1}=(\beta_k I+2\rho X^TX)^{-1}(2\rho X^TX+Y^k+\beta_k Z^{k}).
\label{closeform}
\end{equation} 

Step 2: Computing $Z^{k+1}$. Fix $W^{k+1}$ and $Y^k$, and minimize $L(Y^k, Z, W^{k+1}, \beta_k)$ as follows:
\begin{equation}
\label{variablez}
\begin{split}
Z^{k+1}=& \argmin_Z \hspace{.3cm} L(Y^k, Z, W^{k+1}, \beta_k)\\
            =& \argmin_Z \hspace{.3cm}\log\hspace{.1cm} det(I+Z^TZ)+\\
&\frac{\beta_k}{2}\left\|Z-(W^{k+1}-\frac{1}{\beta_k}Y^k)\right\|_F^2.
\end{split}
\end{equation} 
This can be converted to a scalar minimization problem due to the following theorem. As we notice, this can also be rewritten as s special case of the problem in a recent work \cite{lu2015generalized}.

\begin{theorem}
For unitarily invariant function $F(Z)=f\circ\sigma(Z)$, assuming SVD of $A\in\mathcal{R}^{m\times n}$ is $A = U \Sigma_A V^T$, $\Sigma_A=diag(\{\sigma_{i, A}\}_{i=1}^{ \min(m, n)})$, the optimal solution to the following problem 
\begin{equation}
\min_{Z}F(Z)+\frac{\beta}{2}\|Z-A\|_F^{2}
\label{eq:Zobj}
\end{equation}
 is $Z^* = U \Sigma_Z^* V^T$, with $\Sigma_Z^* = diag( \{  \sigma_i^*   \}_{i=1}^{ \min(m, n) } )$ obtained by solving scalar minimization problems
\begin{equation}
\label{scalar}
 \sigma_i^*  = \arg\min_{\sigma_i} f(\sigma_i) + \frac{\beta}{2} (\sigma_i - \sigma_{i, A})^2, \hspace{.1cm} i=1, \cdots, \min(m, n).
\end{equation} 
 
\end{theorem}
\begin{proof}
Let $A=U\Sigma_{A}V^T$ be SVD of $A$, then $\Sigma_{A}=U^TAV$. Denoting $X=U^TZV$ which has exactly the same signular values as $Z$, i.e., $\Sigma_X=\Sigma_Z$, we have
\begin{flalign}
%\label{eq:X_proof}
  &F(Z)+\frac{\beta}{2}\|Z-A\|_F^{2}&\label{lower}\\
&= F(X)+\frac{\beta}{2}\|X-\Sigma_A\|_{F}^{2},&\label{unitary}\\
&= F(\Sigma_X)+\frac{\beta}{2}\|X-\Sigma_A\|_{F}^{2},&\label{inva}\\
&= F(\Sigma_X)+\frac{\beta}{2}\left(\|X\|_F^2+\|\Sigma_A\|_F^2-2\left\langle X,\Sigma_A\right\rangle \right),&\label{von}\\
&\geq  F(\Sigma_X)\!+\!\frac{\beta}{2}\left(\|\Sigma_X\|_F^2\!+\!\|\Sigma_A\|_F^2\!-\!2\left\langle \Sigma_X,\Sigma_A\right\rangle \right),&\label{eight}\\
&= F(\Sigma_X)+\frac{\beta}{2}\|\Sigma_X-\Sigma_A\|_F^{2},&\label{nine}\\
&= F(\Sigma_Z)+\frac{\beta}{2}\|\Sigma_Z-\Sigma_A\|_F^{2},&\label{ten}\\
&= \sum_i\left[f(\sigma_i)+\frac{\beta}{2}(\sigma_i-\sigma_{i,A})^2\right],&\\
&\ge \sum_i f(\sigma_i^*) + \frac{\beta}{2} (\sigma_i^* - \sigma_{i, A})^2. 
\end{flalign}
In the above, (\ref{unitary}) holds because the Frobenius norm is unitary invariant; (\ref{inva}) holds because $F(Z)$ is unitary invariant; (\ref{von}) is true by von Neumann's inequality; and (\ref{ten}) holds as $\Sigma_X = \Sigma_Z$. The inequality between (\ref{unitary}) and (\ref{nine}) can also be obtained by the Hoffman-Wielandt inequality. Therefore, (\ref{ten}) is a lower bound of (\ref{lower}), where $\Sigma_Z^*$ is obtained by minimizing (\ref{ten}). Note that the equality in (\ref{eight}) is attained if $X = \Sigma_X$. 
Because $\Sigma_Z = \Sigma_X = X = U^T Z V$, the SVD of $Z$ is $Z = U \Sigma_Z V^T$, which is the minimizer of problem (\ref{eq:Zobj}). 
Hence the proof is completed.  

\end{proof}

The first-order optimality condition is that the gradient of (\ref{scalar}) with respect to each singular value should vanish.  
Thus for subproblem (\ref{variablez}), we have 
\begin{equation}
\label{svdf}
\frac{2\sigma_i}{1+\sigma_{i}^{2}}+\beta_k (\sigma_{i}-\Sigma_{i}^{k})=0,\hspace{0.1cm} s.t. \hspace{0.1cm}\sigma_i \ge  0, \hspace{.1cm} for\hspace{0.1cm} i=1, ..., n,\\
\end{equation} 
where SVD of $W^{k+1}-\frac{1}{\beta_k}Y^k$ is $Udiag(\{\Sigma_{i}^k\}_{i=1}^n)V^T$.
The above equation is cubic and gives three roots. In addition, we need to enforce the nonnegativity of $\sigma_i$. 
It is easily seen that there exists at least one nonnegative root. And there is a unique minimizer $\sigma_i^*\in[0, \Sigma_i^k)$ if $\beta_k>1/4$. Finally, we obtain the update of $Z$ variable with $Z^{k+1}=U diag(\sigma_{1}^{k*}, ..., \sigma_{n}^{*}) V^T $.

Step 3: Computing $Y^{k+1}$. Fix $Z^{k+1}$ and $W^{k+1}$, and then we calculate $Y^{k+1}$ as follows:
\begin{equation}
Y^{k+1}=Y^{k}+\beta_k(Z^{k+1}-W^{k+1}).
\label{multi}
\end{equation}

Step 4: Updating $\beta_{k+1}$ as $\beta_{k+1}=\gamma \beta_k$. 
 The complete procedure is summarized in Algorithm 1. 

Problem (\ref{problemd}) is nonconvex. It is difficult to give a rigorous mathematical argument for convergence to a (local) optimum. We will provide a theoretical proof that our algorithm converges to an accumulation point and this accumulation point is a stationary point. Our empirical experiments confirm the convergence of the proposed method on the benchmark datasets. The experimental results are promising, despite that the solution obtained by the proposed optimization method  may be a local optimum. 

\subsection{Affinity graph matrix construction}
Now we will construct an affinity matrix $W$ for subspace clustering. Optimal $Z^*$ may not accurately describe the relationship between samples if the data is severely corrupted. Therefore, in general, it is not a good idea to construct $W$ by directly using $Z^*$. In the spirit of \cite{liu2013robust,lauer2009spectral}, we construct an affinity matrix in the following way. 

Assuming the skinny SVD of $Z^*$ is $U^*\Sigma^*(V^*)^T$, we define $M=U^*(\Sigma^*)^{1/2}$ and $N=(\Sigma^*)^{1/2}(V^*)^T$. Based on the weighted eigen-vector matrix $M$ or $N$, we construct an  affinity matrix $W$ as follows:
\begin{equation}
 W_{ij}=(\frac{m_i^Tm_j}{\left\|m_i\right\|_2\left\|m_j\right\|_2})^{2\alpha}\hspace{0.2cm} or \hspace{.2cm}  W_{ij}=(\frac{n_i^Tn_j}{\left\|n_i\right\|_2\left\|n_j\right\|_2})^{2\alpha}\hspace{0.2cm},
\label{affinity}
\end{equation}
where $m_i$ ($n_i$) and $m_j$ ($n_j$) represent the $i$-th and $j$-th columns (rows) of $M$ ($N$), respectively, and parameter $\alpha \in {\mathcal{N}}$ tunes the sharpness of the affinity between two points, with $\alpha>1$ helping separate the clusters. When $\alpha$ increases, while the between-cluster separability can be increased, the intra-cluster cohesiveness would nevertheless be degraded. Thus, a suitable $\alpha$ needs to balance within-cluster cohesiveness and between-cluster separability. In this paper, we set $\alpha$ to be 2. Then we have the same post-processing as LRR\footnote{For LRR, we use equation (12) in \cite{liu2013robust} rather than (3) to construct $W$. We also confirmed with an author of \cite{liu2013robust}, the power 2 of equation (12) is a typo, it should be 4.}. As $U^*$ or $V^*$ spans the principal directions of $Z^*$, we employ the angle information, or powered correlation coefficients of the examples, because their lengths may be affected significantly by the noise or outliers in the data.  

Now using the resultant affinity matrix, we can apply spectral clustering algorithm to do segmentation. In this paper, we simply perform NCuts \cite{shi2000normalized} on $W$. The proposed subspace clustering procedure is summarized in Algorithm 2.

\begin{algorithm}[t]
\caption{\textbf: The SCLD Algorithm}
\textbf{Input:} data matrix $X$, number of subspaces $k$, parameters $\rho>0$, $\gamma>1$, \text{and }$\beta_0>0$. 
\begin{algorithmic}[1]
\STATE Obtain $Z^*$ from Algorithm 1.
\STATE Compute the skinny SVD $Z^*=U^* \Sigma^* (V^*)^T$.
\STATE Calculate $M=U^*(\Sigma^*)^{1/2}$ or $N=(\Sigma^*)^{1/2}(V^*)^T$.
\STATE Construct the affinity graph matrix $W$ by (\ref{affinity}).
\STATE Apply $W$ to perform NCuts.

\end{algorithmic}
\end{algorithm} 

\section{Convergence Analysis}
\label{convergence}
In this section, we give the convergence analysis for Algorithm 1. We will show that our optimization algorithm attains at least one stationary point of problem (\ref{newproblem}).
We first rewrite the objective function of (\ref{newproblem}) as
\begin{flalign}
G(Z,W)&=F(Z)\!+\!\rho\|X-XW\|_{F}^2\hspace{.1cm}s.t.\hspace{.1cm} Z=W, \label{first}\\
H(Z,W,Y)&=G(Z,W)+\left\langle Z-W,Y\right\rangle,  \label{second}\\
L(Z,W,Y,\beta)&=H(Z,W,Y)+\frac{\beta}{2}\|Z-W\|_{F}^2. \label{eq:AugLag}
\end{flalign}

\begin{lemma}
The sequence $\{Y_{k}\}$ is bounded.
\end{lemma}
\begin{proof}
To minimize $Z$ at step $k+1$, the optimal $Z_{k+1}$ needs to satisfy the first-order optimality condition
\begin{equation}
\begin{split}
& \nabla_Z L\left(Z,W_{k+1},Y_{k},\beta_{k}\right)|_{Z_{k+1}}\\
=&\nabla_Z F\left(Z\right)|_{Z_{k+1}}+\beta_{k}\left(Z_{k+1}+\frac{1}{\beta_k}Y_{k}-W_{k+1}\right)=0.
\end{split}
\end{equation}
Note that the updating rule for $Y$ is 
\begin{equation}
Y_{k+1}=Y_{k}+\beta_{k}\left(Z_{k+1}-W_{k+1}\right),
\end{equation}
thus $\nabla_Z F\left(Z\right)|_{Z_{k+1}}+Y_{k+1}=0$.
We know from (\ref{theorem}) that
\begin{equation}
\begin{split}
 &\nabla_Z F\left(Z\right)|_{Z_{k+1}}\\
=& U diag\left(\frac{2\sigma_1}{1+\sigma_{1}^2},..., \frac{2\sigma_n}{1+\sigma_{n}^2}\right)V^T,
\end{split}
\end{equation}
and $0\leq\frac{2\sigma_i}{1+\sigma_{i}^2}\leq 1$, so $\nabla_Z F\left(Z\right)|_{Z_{k+1}}$ is bounded. Then it is seen that $Y_{k+1}$, i.e., $\{Y_k\}$ is bounded.
\end{proof}

\begin{lemma}
\label{lemma:ZW_bounded}
$\{W_k\}$ and $\{Z_k\}$ are bounded if $\sum \frac{\beta_{k+1}}{\beta_k^2}<\infty$ and $\sum \frac{1}{\beta_k}<\infty$.
\end{lemma}
\begin{proof}
\begin{equation}
\begin{split}
& L\left(Z_{k},W_{k},Y_{k},\beta_{k}\right)\\
=&L\left(Z_{k},W_{k},Y_{k-1},\beta_{k-1}\right)+\frac{\beta_k-\beta_{k-1}}{2}\|Z_k-W_k\|_F^2+\\
&Tr((Y_k-Y_{k-1})(Z_k-W_k))\\
=& L\left(Z_{k},W_{k},Y_{k-1},\beta_{k-1}\right)+\frac{\beta_k+\beta_{k-1}}{2\beta_{k-1}^2}\|Y_k-Y_{k-1}\|_F^2.
\end{split}
\end{equation}

Thus,
\begin{equation}
\begin{split}
& L\left(Z_{k+1},W_{k+1},Y_{k},\beta_{k}\right)\\
\leq & L\left(Z_{k},W_{k+1},Y_{k},\beta_{k}\right),\\
\leq & L\left(Z_{k},W_{k},Y_{k},\beta_{k}\right),\\
\leq & L\left(Z_{k},W_{k},Y_{k-1},\beta_{k-1}\right)+\frac{\beta_k+\beta_{k-1}}{2\beta_{k-1}^2}\|Y_k-Y_{k-1}\|_F^2,\\
\leq &...\\
\leq & L\left(Z_{1},W_{1},Y_{0},\beta_{0}\right)+\sum_{i=1}^{k}\frac{\beta_i+\beta_{i-1}}{2\beta_{i-1}^2}\|Y_i-Y_{i-1}\|_F^2.\\
\end{split}
\end{equation}
Since the second term in above inequality is finite, \\
$L\left(Z_{k+1},W_{k+1},Y_{k},\beta_{k}\right)$ is bounded.
We can rewrite\\
 $L\left(Z_{k+1},W_{k+1},Y_{k},\beta_{k}\right)$ as 
\begin{equation}
\begin{split}
& L(Z_{k+1},W_{k+1},Y_{k},\beta_{k})+\frac{1}{2\beta_k} ||Y_k||_F^2\\ %=L(Z,W,Y)+\frac{\beta}{2}\|Z-W\|_{F}^2\\
=& F(Z_{k+1})+\rho\|X-XW_{k+1}\|_F^2+\\
&\frac{\beta_k}{2}\|Z_{k+1}-W_{k+1}+\frac{1}{\beta_k}Y_k\|_F^2.
\label{bound}
\end{split}
\end{equation}
Because $L(Z_{k+1},W_{k+1},Y_{k},\beta_{k})$ and $\frac{1}{\beta_k} ||Y_k||_F^2$ are bounded and each term on the right hand side of the equation (\ref{bound}) is nonnegative, each term will be bounded.
$F(Z_{k+1})=\sum_{i} \log (1+\sigma_{i}^2(Z_{k+1}))$ being bounded implies that all singular values of $Z_{k+1}$ are bounded and $Z_{k+1}$ is bounded.
Since $\frac{1}{\beta_k} (Y_{k+1} - Y_k) = Z_{k+1} - W_{k+1}$, clearly  we have bounded $W_{k}$. Therefore $\{W_k\}$ and $\{Z_k\}$ are bounded.
\end{proof}

\begin{theorem}
$\{Y_k,W_k,Z_k\}$ has at least one accumulation point $\{Y^*,W^*,Z^*\}$, and $\{ W^*, Z^* \}$ is a stationary point of optimization problem (\ref{newproblem}) with the assumption that $\lim\limits_{k\rightarrow\infty}\beta_{k-1}(Z_{k}-Z_{k-1})\rightarrow 0$.
\end{theorem}

\begin{proof}
$\{Y_k,W_k,Z_k\}$ is a bounded sequence, hence by the Bolzano-Weierstrass theorem, there must be at least one accumulation point, which is denoted by $\{Y^*,W^*,Z^*\}$.
Without loss of generality, we assume that 
$\{  Y_k, W_k, Z_k \}$ itself converges to 
$\{ Y^*, W^*, Z^* \}$. Next, we prove that this accumulation point is a stationary point of problem (\ref{first}).
As $Y_{k}=Y_{k-1}+\beta_{k-1} (Z_{k}-W_{k})$, we have $Z_{k}-W_{k}=\frac{1}{\beta_{k-1}}(Y_{k}-Y_{k-1})$. Because $\beta_{k-1}\rightarrow \infty$ and $\{Y_{k}\}$ is bounded, we get $Z_{k}-W_{k}\rightarrow 0$, i.e., $Z^*=W^*$. By first-order optimality condition and the definition of $Z_{k}$, we have 
$\nabla_Z F(Z)|_{Z_{k}} + Y_{k-1} + \beta_{k-1} (Z_{k} - W_{k})  = \nabla_Z F\left(Z\right)|_{Z_{k}}+Y_{k}=0$. Let $k\rightarrow\infty$, we get $\nabla_Z F\left(Z\right)|_{Z^*}+Y^*=0$. At the $k$th step, $W_{k}$ satisfies $(2\rho X^TX+\beta_{k-1}I)W_{k}=2\rho X^TX+\beta_{k-1}Z_{k-1}+Y_{k-1}$, i.e., $2\rho X^TX(W_{k}-I)=\beta_{k-1}Z_{k-1}-\beta_{k-1}W_{k}+Y_{k-1}=\beta_{k-1}(Z_{k}-W_{k})+\beta_{k-1}(Z_{k-1}-Z_{k})+Y_{k-1}=\beta_{k-1}(Z_{k-1}-Z_{k})+Y_{k}$. With the assumption that $\beta_{k-1}(Z_{k}-Z_{k-1})\rightarrow 0$ \cite{nie2014new}, we get $2\rho X^TX(W^*-I)=Y^*$.

Now we can see that $\{Y^*, W^*, Z^*\}$ satisfies the KKT conditions of $L(W, Z, Y )$ and thus $\{  W^*, Z^* \}$  is a stationary point of (\ref{newproblem}).
\end{proof}

\begin{table*}[!htbp]
\small
\setlength{\abovecaptionskip}{0pt}
\setlength{\belowcaptionskip}{10pt}
\centering
\caption{Parameter settings of different algorithms.}
\label{parameters}
\begin{tabular}{|c|c|c|c|c|}
\hline
\multirow{2}{*}{Method} & \multicolumn{2}{c|}{Face clustering} & \multicolumn{2}{c|}{Motion segmentation} \\
\cline{2-5 }
 & Scenario 1 & Scenario 2 &\multicolumn{2}{c|}{} \\
\hline
  LRR & \multicolumn{2}{c|}{$ \lambda = 0.18 $} & \multicolumn{2}{c|}{$ \lambda = 4 $} \\
\hline
  LSA & \multicolumn{2}{c|}{$K=3, d=5$} &\multicolumn{2}{c|}{ $K=8, d=5$}    \\
\hline  
SSC & ${\lambda _e} = {{8} \mathord{\left/
 {\vphantom {{800} {{\mu _e}}}} \right.
 \kern-\nulldelimiterspace} {{\mu _e}}}$ & ${\lambda _e} = {{20} \mathord{\left/
 {\vphantom {{800} {{\mu _e}}}} \right.
 \kern-\nulldelimiterspace} {{\mu _e}}}$ & \multicolumn{2}{c|}{${\lambda _z} = {{800} \mathord{\left/
 {\vphantom {{800} {{\mu _z}}}} \right.
 \kern-\nulldelimiterspace} {{\mu _z}}}$} \\
\hline
  LRSC & $\tau  = 0.4,\alpha  = 0.045$ &  $\tau  = 0.045,\alpha  = 0.045$  & \multicolumn{2}{c|}{$\tau  = 420,\alpha  = 3000\,or\,\alpha  = 5000$ } \\
\hline
  SCLD & $\rho  = 0.08$ &  $\rho  = 0.03$  & \multicolumn{2}{c|}{$\rho  = 55$ } \\
\hline

\end{tabular}
\end{table*}

\section{Experiments and Analysis}
In this section, we conduct experiments on the subspace clustering task with both synthetic and real data.
\label{experiment}
\subsection{Experiments with Synthetic Data}
We construct 5 independent subspaces whose bases $\{U_i\}_{i=1}^5$ are generated by a random rotation matrix $R$ through $U_{i+1}=RU_i$, $1\leq i\leq4$, where $U_1\in \mathcal{R}^{100\times4}$ is a random orthogonal matrix \cite{liu2010robust}. We sample 20 data vectors from each subspace by $X_j=U_j T_j$, $1\leq j\leq5$, where $T_j$ is a $4\times 20$ iid $\mathcal{N}(0,1)$ matrix. Some data vectors are randomly chosen to corrupt; for example, for a data vector $x$, it is corrupted by adding Gaussian noise with zero mean and variance $0.2\|x\|$. We then use SCLD to segment the data into 5 clusters. Subspace clustering error rate defined as 
$
\small
\frac{\text{\# of misclassified points}}{\text{total \# of points}}
$
 is used to assess the performance. We report the clustering error rate (averaged from 30 trials) with different corruption levels in Figure \ref{syn}. Without any corruption, SCLD can cluster all data points correctly.
\begin{figure}
\centering
    \includegraphics[width=.48\textwidth]{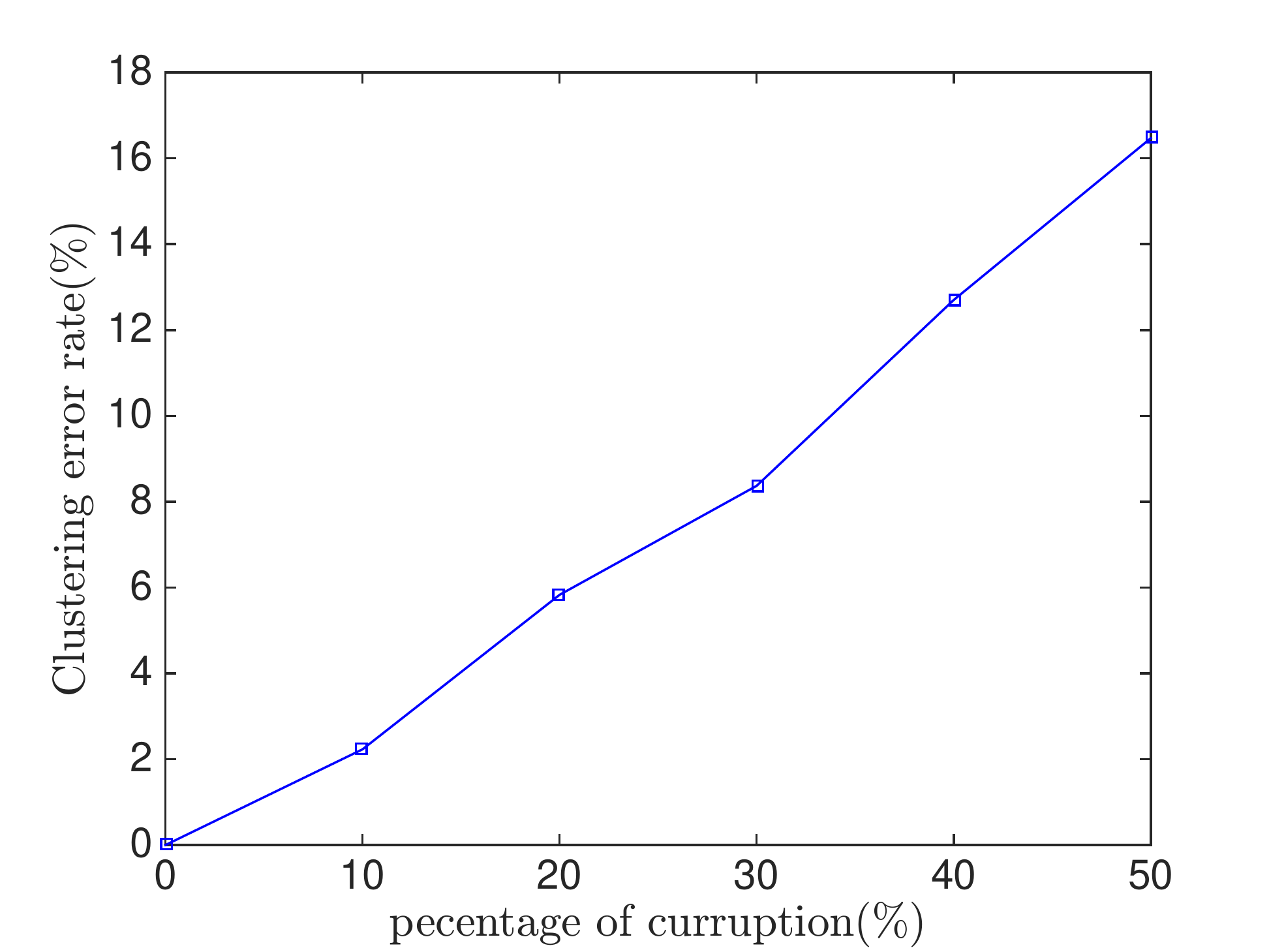}
\caption{The clustering error rate with different percentage of corruption on synthetic data. The parameter $\rho$ is tuned to obtain the best performance.}
\label{syn}
\end{figure}

\subsection{Experiments with Real Data}

In this section, we evaluate the effectiveness and robustness of SCLD on benchmark datasets, Extended Yale B (EYaleB) \cite{georghiades2001few,lee2005acquiring} and Hopkins 155 \cite{tron2007benchmark}. We compare the proposed method SCLD with several state-of-the-art subspace clustering algorithms: LRR \cite{liu2013robust}, SSC \cite{elhamifar2013sparse}, LRSC \cite{favaro2011closed,vidal2014low}, and local subspace affinity (LSA) \cite{yan2006general}. For these methods, we use the parameters given by the respective authors. For our method, we also tune $\rho$ to obtain the best performance. Generally, $\rho$ should be relatively large if the data are slightly corrupted. $\beta$ and $\gamma$ have little influence on the clustering results, so we just set  $\beta_0=0.3$ to ensure the unique of minimizer and use $\gamma=1.1$ empirically. Other parameters are shown in Table \ref{parameters}. The experiments are conducted on Window 7 with 16 GM memory and Intel Core i5-2300 CPU. 
\begin{figure*}
\centering
    \includegraphics[width=\textwidth]{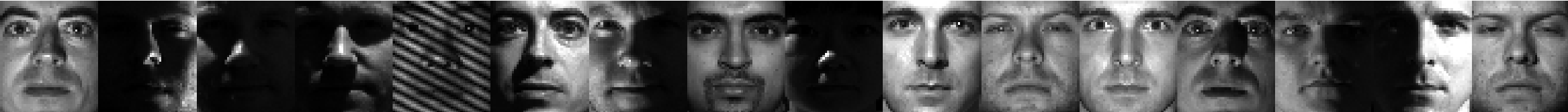}
\caption{Sample images from the Extended Yale B database.}
\label{faceimg}
\end{figure*}

\subsubsection{Face Clustering}
Face clustering is to cluster a set of face images from multiple individuals in a hope to reveal the identity of these individuals.  
EYaleB Database includes 2414 frontal images of 38 individuals. For each individual, the images are taken under 64 lighting conditions and can be described by a low-dimensional subspace \cite{basri2003lambertian}. The images are resized to 48$\times$42 pixels and each vectorized image is regarded as a data point. Fig. \ref{faceimg} shows some example images from the database.

\paragraph{First Experiment Scenario}
\begin{table}[ht]
\caption{Clustering error rate on the first 10 classes of EYaleB.}
%\centering
\begin{tabular}{c c c c c c }
\hline
Method & LRR&SSC&LSA&LRSC&SCLD  \\
\hline
error rate (\%)&20.94&35&59.52&35.78&\textbf{3.59} \\ %[lex]
\hline
\end{tabular}
\label{table:firstyale}
\end{table}
\noindent As done in \cite{liu2010robust}, we test the algorithms on the first 10 classes of  EYaleB, which consists of 640 frontal face images. More than half of the images are corrupted by shadow and noise. We use this heavily corrupted data to test the effectiveness of our method. As shown in Table \ref{table:firstyale}, SCLD significantly enhances the performance. Specifically, it improves the clustering accuracy by at least $17\%$ when compared to the other algorithms. Since the only difference between our approach and LRR is rank approximation, this improvement is due to LogDet. 

\paragraph{Second Experiment Scenario}
\begin{table}[ht]
\caption{Clustering error rates (\%) on EYaleB.}
%\centering
\begin{tabular}{cccccc}
\hline\noalign{\smallskip}
Method & LRR&SSC&LSA&LRSC& SCLD  \\
\noalign{\smallskip}\hline\noalign{\smallskip}
2 Subjects &&&&&\\
 Mean&2.54&\textbf{1.86}&32.80&5.32&2.79 \\
Median&0.78&\textbf{0.00}&47.66&4.69&0.78 \\
\noalign{\smallskip}\hline\noalign{\smallskip}
3 Subjects &&&&&\\
 Mean&4.21&\textbf{3.10}&52.29&8.47&3.72 \\
Median&2.60&\textbf{1.04}&50.00&7.81&1.56 \\
\noalign{\smallskip}\hline\noalign{\smallskip}
5 Subjects &&&&&\\
Mean&6.90&\textbf{4.31}&58.02&12.24&4.83 \\
Median&5.63&\textbf{2.50}&56.87&11.25&\textbf{2.50} \\
\noalign{\smallskip}\hline\noalign{\smallskip}
8 Subjects &&&&&\\
 Mean&14.34&5.85&59.19&23.72&\textbf{5.45} \\
Median&10.06&4.49&58.59&28.03&\textbf{3.52} \\
\noalign{\smallskip}\hline\noalign{\smallskip}
10 Subjects &&&&&\\
Mean&22.92&10.94&60.42&30.36&\textbf{6.25} \\
Median&23.59&5.63&57.50&28.75&\textbf{4.84} \\
\noalign{\smallskip}\hline
\end{tabular}
\label{table:yale2nd}
\end{table}
For a fair comparison, we have followed the experimental setup of \cite{elhamifar2013sparse}. We divide the 38 subjects into four groups: subjects 1 to 10, 11 to 20, 21 to 30, and 31 to 38. We consider all choices of $n\in\lbrace2, 3, 5, 8, 10\rbrace$ subjects for the first three groups. For the last group, we consider all choices of $n\in\lbrace2, 3, 5, 8\rbrace$. We implement our subspace clustering algorithm on each set of $n$ subjects. For all experiments, the stopping criterion for $Z$ is triggered by a relative difference of $10^{-5}$ between two successive iterations, or by a maximum of 100 iterations.    

The results are presented in Table \ref{table:yale2nd}. For other methods, we cited the results from Table 5 of paper \cite{elhamifar2013sparse}. SCLD consistently has low clustering error rates and is more stable than the other methods whose error rates increase drastically as the number of subjects increases to 8 and 10. As shown in Figure \ref{faceimg}, there are many sparse within-sample outliers in the face images, e.g, shadows. Although LRR uses a regularization term to count for corruptions, the regularization term does not appear to be well suited to EYaleB. LSA has inferior performance possibly because it does not explicitly exploit the low-rank structure of the data.   

\paragraph{Third Experiment Scenario}
In this section, we compare SCLD with other algorithms with RPCA \cite{candes2011robust} as a preprocessing step. In practice, we do not know the clustering of the data beforehand and hence we apply RPCA to the collection of all data points for each trial prior to clustering. As shown in Table \ref{table:yale3rd}, SCLD is still superior to other methods though they apply RPCA to deal with sparse outlying entries. Compared to Table \ref{table:yale2nd}, only the clustering error rates of LRSC reduced in some cases. We can conclude that applying RPCA to all data points simultaneously is not effective to improve clustering performance. This is due to the fact that RPCA seeks a common low-rank subspace, which will decrease the principal angles between subspaces and decrease the distance between data points in different subjects \cite{elhamifar2013sparse}.
\begin{table}[ht]
\caption{Clustering error rates (\%) on EYaleB after applying RPCA simultaneously to all the data in each trial.}
%\centering
\begin{tabular}{cccccc}
\hline\noalign{\smallskip}
Method & LRR&SSC&LSA&LRSC& SCLD  \\
\noalign{\smallskip}\hline\noalign{\smallskip}
2 Subjects &&&&&\\
 Mean&5.72&\textbf{2.09}&32.53&5.67&2.79 \\
Median&3.91&\textbf{0.78}&47.66&4.69&\textbf{0.78} \\
\noalign{\smallskip}\hline\noalign{\smallskip}
3 Subjects &&&&&\\
 Mean&10.01&3.77&53.02&8.72&\textbf{3.72} \\
Median&9.38&2.60&51.04&8.33&\textbf{1.56} \\
\noalign{\smallskip}\hline\noalign{\smallskip}
5 Subjects &&&&&\\
Mean&15.33&6.79&58.76&10.99&\textbf{4.83} \\
Median&15.94&5.31&56.87&10.94&\textbf{2.50} \\
\noalign{\smallskip}\hline\noalign{\smallskip}
8 Subjects &&&&&\\
 Mean&28.67&10.28&62.32&16.14&\textbf{5.45} \\
Median&31.05&9.57&62.50&14.65&\textbf{3.52} \\
\noalign{\smallskip}\hline\noalign{\smallskip}
10 Subjects &&&&&\\
Mean&32.55&11.46&62.40&21.82&\textbf{6.25} \\
Median&30.00&11.09&62.50&25.00&\textbf{4.84} \\
\noalign{\smallskip}\hline
\end{tabular}
\label{table:yale3rd}
\end{table}

\subsubsection{Motion Segmentation}
\begin{figure*}[!htbp]
%\centering
\includegraphics[width=0.24\textwidth]{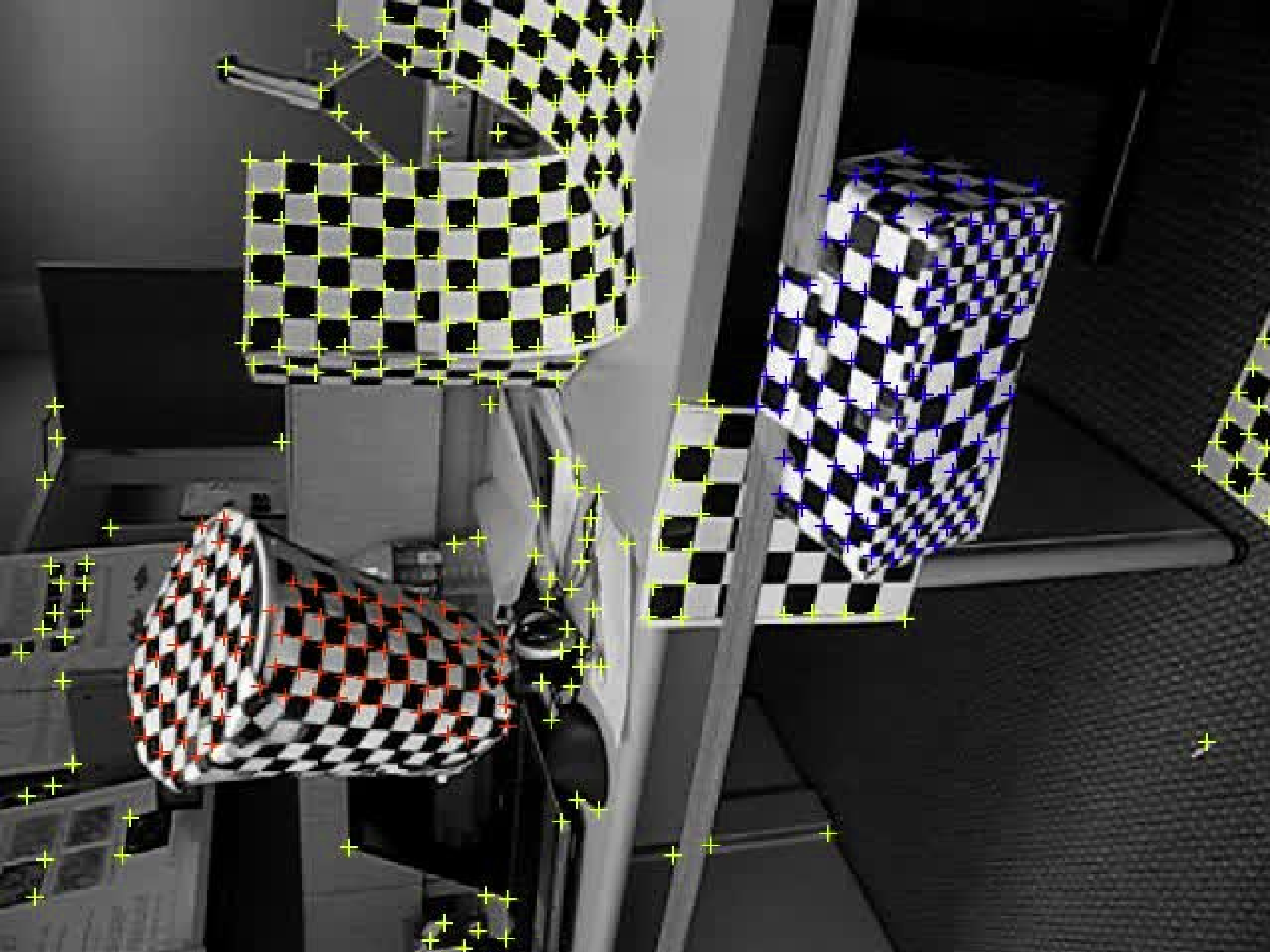}
\includegraphics[width=0.24\textwidth]{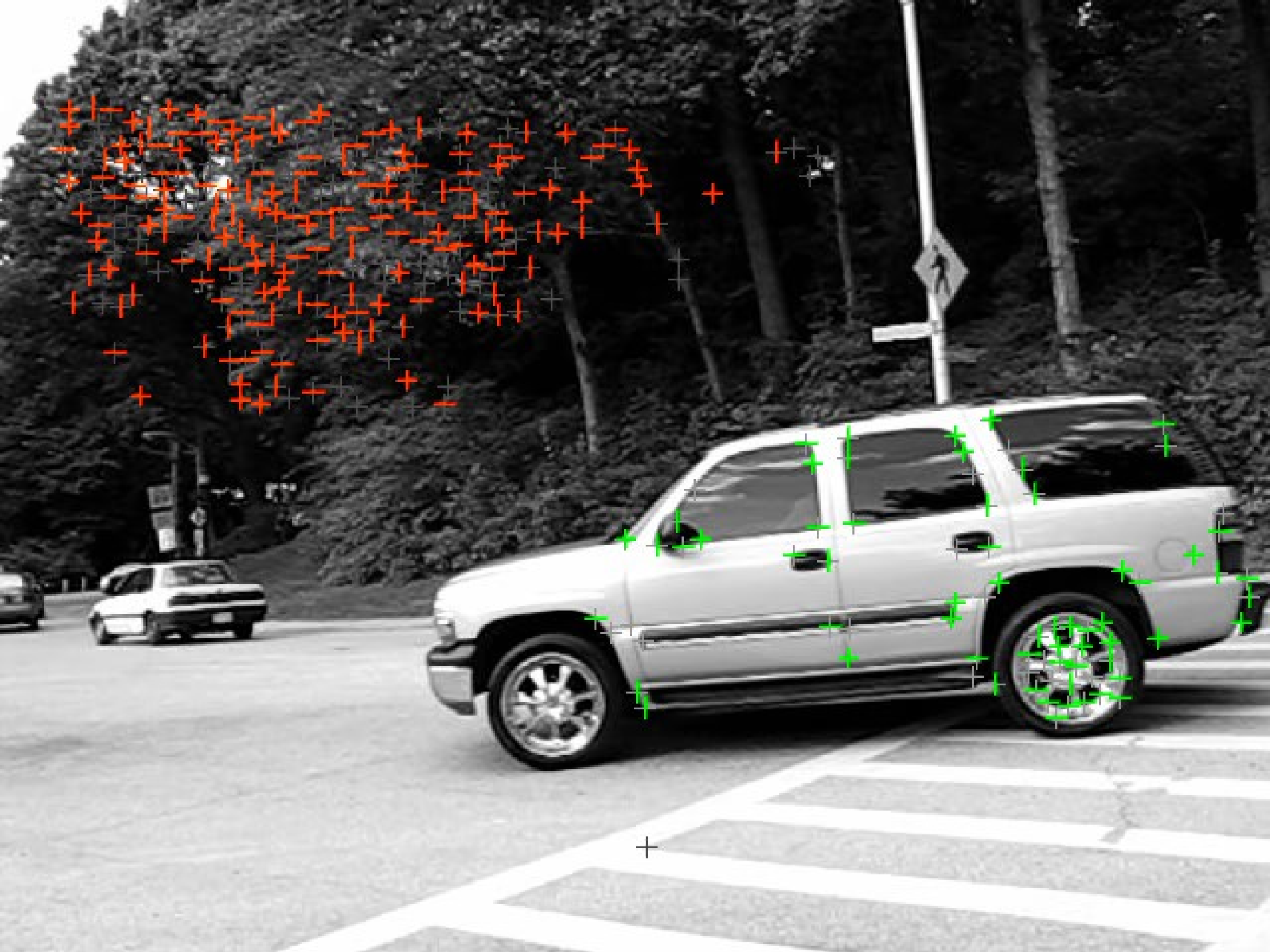}
\includegraphics[width=0.24\textwidth]{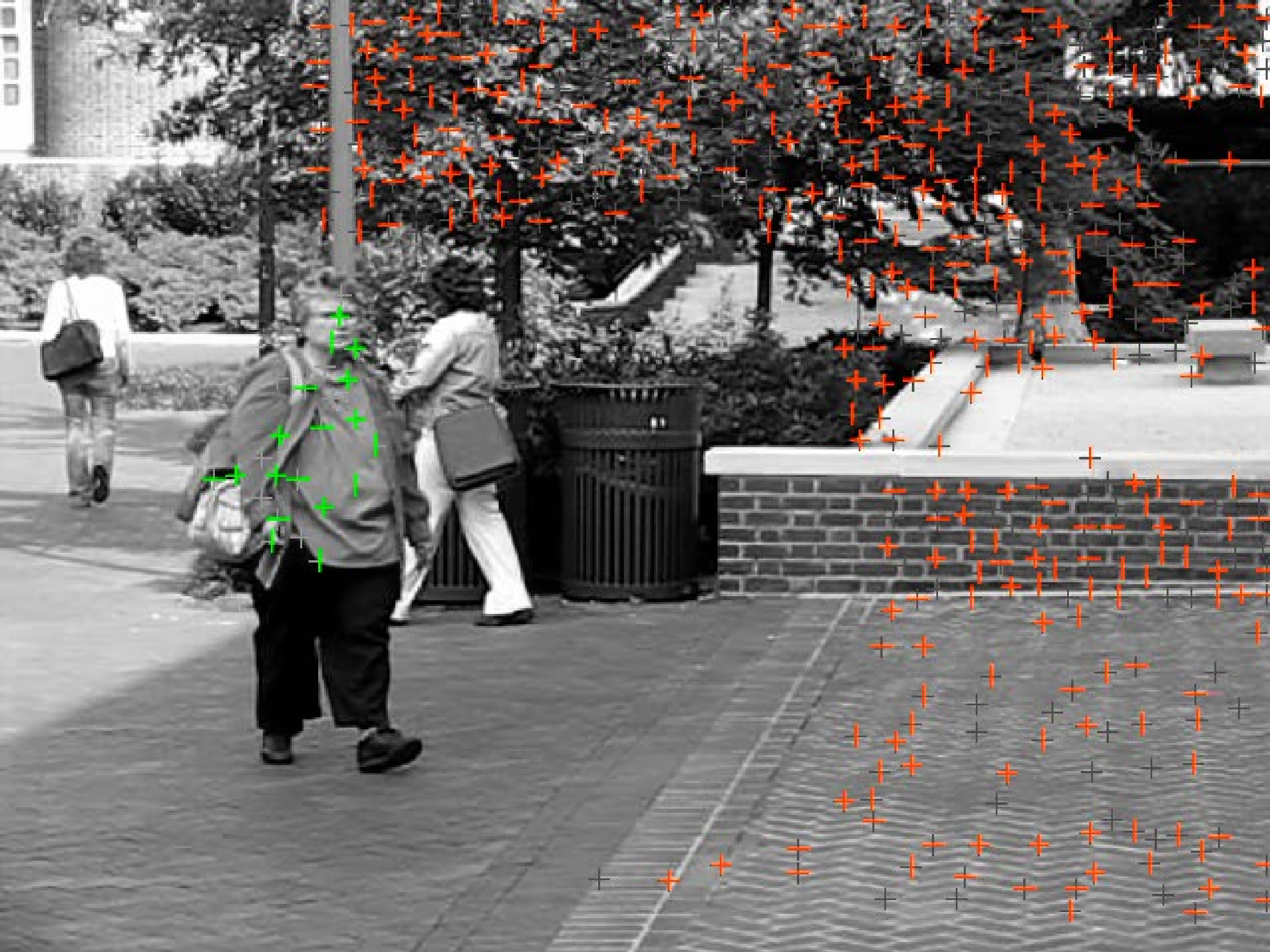}
\includegraphics[width=0.24\textwidth]{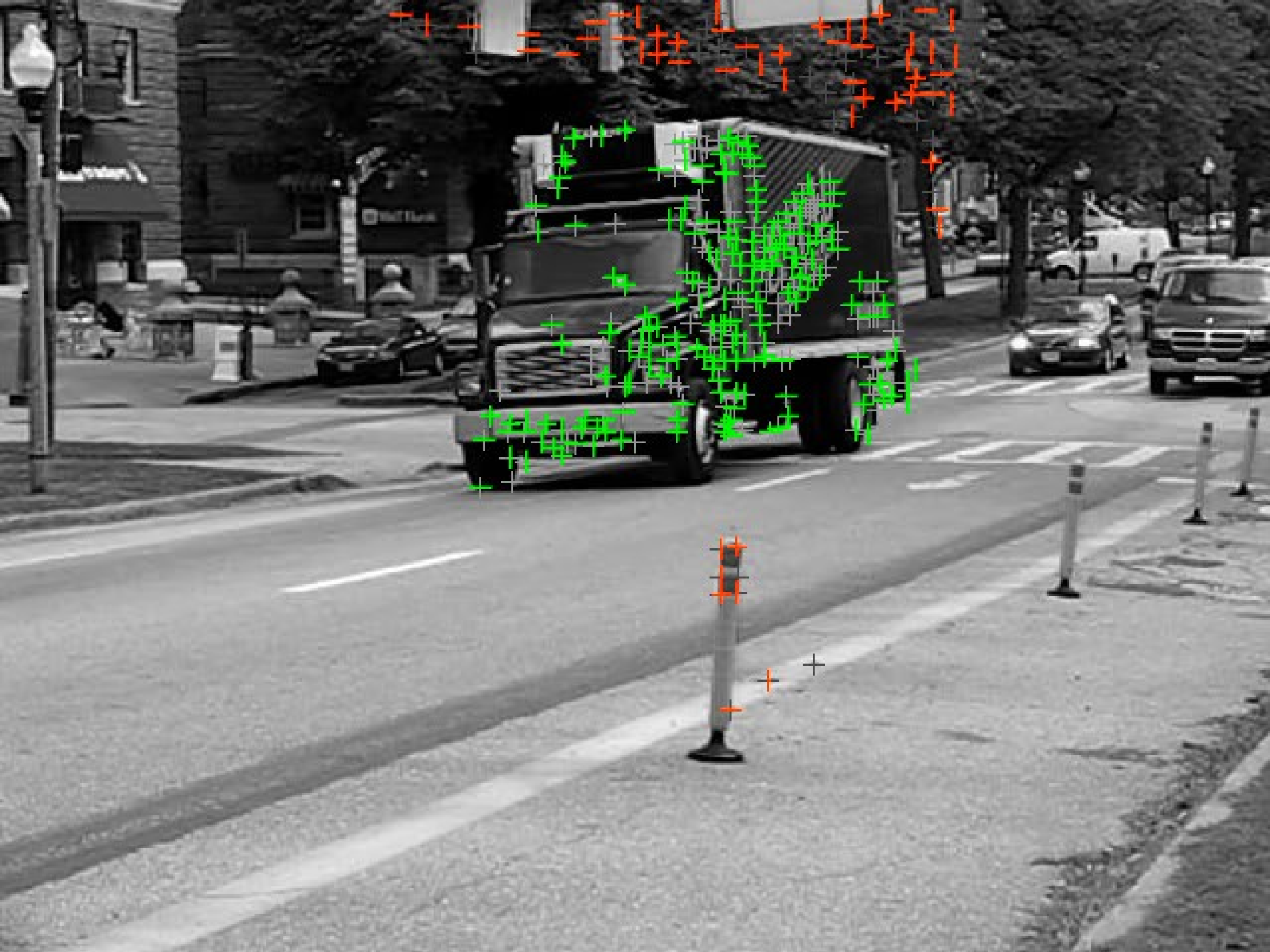}
\caption{Example frames from four video sequences of the Hopkins 155 Database with traced feature points.}
\label{fig:h155example} %% label for entire figure
\end{figure*}
\noindent Motion segmentation is to segment the trajectories associated with $n$ different moving objects into different groups according to their motions in a video sequence.
Because different motions can be treated as different subspaces, we use the Hopkins 155 Dataset to validate SCLD. This dataset is slightly corrupted as shown in Figure 3. It consists of 155 sequences of two or three motions and 1 sequence of 5 motions; the latter is regarded as outlier. Each sequence is regarded as a separate clustering problem. 
\begin{table}[ht]
\caption{Segmentation error rate (\%) on the HopKins 155 Dataset (155 Sequences).}
%\centering
\begin{tabular}{ c c c c c c}
\noalign{\smallskip}\hline
Method & LRR&SSC&LSA&LRSC& SCLD  \\
\noalign{\smallskip}\hline\noalign{\smallskip}
2 Motions&&&&&\\
 Mean&2.13&1.52&4.23&3.69&\textbf{1.31}  \\
Median&\textbf{0.00}&\textbf{0.00}&0.56&0.29&\textbf{0.00}  \\
\noalign{\smallskip}\hline\noalign{\smallskip}
3 Motions &&&&&\\
Mean&4.03&4.40&7.02&7.69&\textbf{3.43} \\
Median&1.43&\textbf{0.56}&1.45&3.80&\textbf{0.56}  \\
\noalign{\smallskip}\hline\noalign{\smallskip}
All &&&&&\\
Mean&2.56&2.18&4.86&4.59&\textbf{1.79} \\
Median&\textbf{0.00}&\textbf{0.00}&0.89&0.60&\textbf{0.00}  \\
\noalign{\smallskip}\hline
Time (sec)&1.30&1.04&3.40&\textbf{0.16}&1.49\\
\noalign{\smallskip}\hline
\end{tabular}
\label{table:motion}
\end{table}

\begin{figure}[!htbp]
\centering
\includegraphics[width=.5\textwidth]{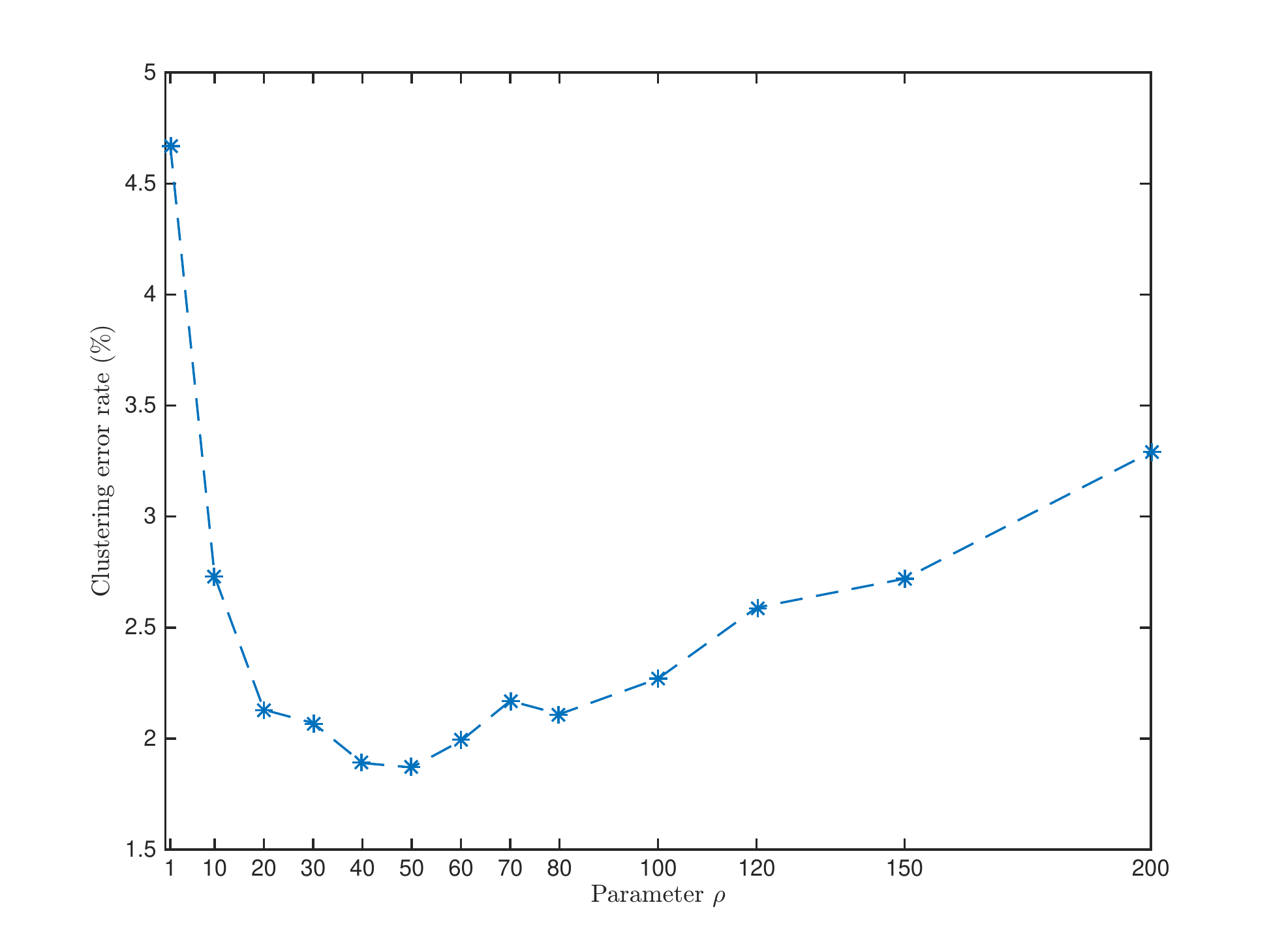}
\caption{Changes in clustering error rate when varying $\rho$.}
\label{fig_parameterinfluence}
\end{figure}
The experimental results are reported in Table \ref{table:motion}. We also used the results in Table 1 of \cite{elhamifar2013sparse}. It can be seen that SCLD produces superior results compared to the other methods. For all 155 sequences, the error rate is as low as 1.79$\%$. If we use all 156 sequences, the overall error rate of our proposed algorithm will be 1.87$\%$. We report the average computation time for every sequence at the bottom of Table \ref{table:motion}. The computational cost of LRSC is much lower than the other methods, while LRR, SSC and SCLD are comparable. 

To testify the influence of parameter $\rho$ in our algorithm, we show the clustering error rates of SCLD for different $\rho$ over all 155 sequences in Figure \ref{fig_parameterinfluence}. As we can see, when $\rho$ was between 1 and 200, the clustering error varied between 1.79$\%$ and 4.67$\%$. This implies that SCLD performs well under a wide range of values of $\rho$. 

To test the dependence of SCLD on initialization, we apply another two different initializations. First, we use the solutions from LRR as initial guess for SCLD. Second, we just generate some random numbers. We find that we can still get the same results. Actually, it is recommended to use convex relaxation solutions as initialization for nonconvex formulations \cite{fan2014strong,zhang2010analysis}.

\section{Conclusion}
In this paper we propose to use a log-determinant function (LogDet) as a rank approximation to recover the low-rank representation of high-dimensional data. When applied to subspace clustering, the proposed algorithm, called SCLD, exploits both global and local structures of the data through the LogDet rank approximation and angle-based affinity matrix. Consequently, it captures more intrinsic information of the data that benefits subspace clustering. Our extensive experimental results show that it outperforms other low-rank representation algorithms based on the nuclear norm. Therefore LogDet appears to be an effective rank approximation function well suited to subspace clustering applications. Although our model is simple and with no explicit modeling of outliers, it is resilient to various corruptions. Our future research will consider modeling corruptions explicitly. 

\label{conclusion}

\begin{acknowledgements}

This work is supported in part  by US National Science Foundation grants IIS 1218712.
\end{acknowledgements}
%\begin{thebibliography}{10}
%\end{thebibliography}
%
\bibliographystyle{spphys}
%\bibliography{scref}
\small{\bibliography{scref} } 
\end{document}